\documentclass[12pt]{elsarticle}

\usepackage[margin=1.4in]{geometry}
\usepackage{algorithm, algorithmic}
\usepackage{amssymb, amsmath, amsthm}
\usepackage{url}
\usepackage{placeins}
\usepackage{xspace}
\usepackage{times}
\usepackage{appendix}
\usepackage{color}
\usepackage{latexsym}

\usepackage{thmtools}
\usepackage{thm-restate}

\usepackage{tikz}
\usetikzlibrary{fit,positioning}

\newtheorem{theorem}{Theorem}
\newtheorem{assumption}[theorem]{Definition}

\newtheorem{definition}[theorem]{Definition}
\newtheorem{lemma}[theorem]{Lemma}

\def\ci{\perp\!\!\!\perp}

\newcommand{\vdag}{\textsc{FindVariableDAG}\xspace}

\journal{International Journal of Approximate Inference}

\begin{document}

\begin{frontmatter}

\title{Learning Structures of Bayesian Networks for Variable Groups\footnote{\copyright 2017. This manuscript version is made available under the CC-BY-NC-ND 4.0 license \url{http://creativecommons.org/licenses/by-nc-nd/4.0/}}}
\author{Pekka Parviainen and Samuel Kaski\\ Helsinki Institute for Information Technology HIIT,\\ Department of Computer Science, \\Aalto University, Espoo, Finland }

\begin{abstract}
Bayesian networks, and especially their structures, are powerful tools for representing conditional independencies and dependencies between
random variables. In applications where related variables form \emph{a  priori} known groups, chosen to represent different ``views'' to or aspects of the same entities, one may be more interested in modeling dependencies between groups of variables rather than between individual variables. Motivated by this, we study prospects of representing relationships between variable groups using Bayesian network structures. We show that for dependency structures between groups to be expressible exactly, the data have to satisfy the so-called groupwise faithfulness assumption. We also show that one cannot learn causal relations between groups using only groupwise conditional independencies, but also variable-wise relations are needed. Additionally, we present algorithms for finding the groupwise dependency structures.
\end{abstract}


\end{frontmatter}

\section{Introduction}

Bayesian networks are representations of joint distributions of random
variables. They are powerful tools for modeling
dependencies between variables. They consist of two parts, the structure and parameters, which together specify the joint distribution. The dependencies and independencies
between variables are implied by the structure of a Bayesian network, which is
represented by a directed acyclic graph (DAG). The parameters specify local conditional probability distributions for each variable.

In practical applications it is common that the analyst does not know
the structure of a Bayesian network \emph{a priori}. However, 
samples from the distribution of interest are commonly available. This
has motivated development of algorithms for learning Bayesian networks
from observational data. There are two main approaches to learning the structure of a Bayesian
network from data: constraint-based and score-based. The {\em
  constraint-based approach} (see, e.g., \cite{pearl00, spirtes00}) relies on
testing conditional independencies between variables. The network is
constructed so that it satisfies the found conditional independencies
and dependencies. In the {\em score-based approach} (see, e.g.,
\cite{cooper92,heckerman95}) one assigns each network a score that
measures how well the network fits the data. Then one tries to find
a network that maximizes the score.
Although the problem is NP-hard
\citep{chickering96}, there exist plenty of exact algorithms
\citep{cussens11, jaakkola10, silander06} as well as theoretically
sound heuristics \citep{aliferis10a,chickering02b}. Learning the parameters given the structure is rather straightforward and thus we concentrate on structure learning.

Bayesian networks model dependencies and independencies between
individual variables. However, sometimes the relationships between groups
of variables are even more interesting. An example is multiple
different measurements of expression of the same genes, made with
multiple measurement platforms, but the goal being to find
relationships between the genes and not of the measurement
platforms. The measurements of each gene would here be the groups.
Another example is measurements of expression of individual genes, with
the goal of the analysis being to understand cross-talk between
pathways consisting of multiple genes, or more generally,
relationships on a higher level of a hierarchy tree in hierarchically
organized data. Here the pathways would be the groups. In both cases,
a Bayesian network for variable groups would directly address the
analysis problem, and would also have fewer variables and hence be
easier to visualize.

More generally, the setup matches multi-view learning where data
consist of multiple ``views'' to the same entity, multiple aspects of
the same phenomenon, or multiple phenomena whose relationships we want
to study. For these setups, a Bayesian network for variable groups can
be seen as a dimensionality reduction technique with which we extract
interesting information from a larger, noisy data set. Note that our model is targeted for a very specific application, that is, on learning conditional independencies between known variable groups. It is not a general-purpose dimensionality reduction technique such as, say, PCA.

While the structure learning problem is well-studied for individual
variables, knowledge about modeling relationships between variable groups
using the Bayesian network framework is scarce. 
Motivated by this, we study prospects of learning Bayesian network structures for variable groups. In summary, while Bayesian networks for variable groups can be learned under some conditions, strong assumptions are required and hence they have limited applicability. 

We start by exploring theoretical possibilities and limitations
for learning Bayesian networks for variable groups. First, we
show that in order to be able to learn a structure that expresses
exactly the conditional independencies between variable groups, the
distribution and the groups need to together satisfy a condition that
we call groupwise faithfulness (Section~\ref{sec:gf}); our simulations suggest that this is a rather strong assumption. 
Then, we study possibilities of
finding causal relations between variable groups. It turns out that
one can draw only very limited causal conclusions based on only the
conditional independencies between groups
(Section~\ref{sec:causality}), and hence also dependencies between the
individual variables are needed. 

We introduce methods for learning Bayesian network
structures for variable groups. First, it is possible to learn a structure directly using
conditional independencies or local scores between groups
(Section~\ref{sec:direct}). However, this approach suffers from needing
lots of data. For the second approach, we observe that if all
conditional independencies between individual variables are known, one
can infer the conditional independencies between groups. The second
approach is to construct a Bayesian network for individual variables
and then to infer the structure between groups
(Section~\ref{sec:var_net}). The third approach is to learn structures for both individual variables and groups simultaneously (Section~\ref{sec:simultaneous}). 
Finally, we evaluate the algorithms in practice (Section~\ref{sec:experiments}). 
Our results suggest that the second and third approaches are more accurate.

\subsection{Related Work}

We are not aware of any work with close resemblance with this
study, but there have been some efforts to solve related problems. Next, we will briefly introduce some related and explain why we have not based our work on them. 

\sloppy
Object-oriented Bayesian networks \citep{koller97} are a generalization of Bayesian networks and enable representing groups of variables as objects. Hierarchical Bayesian networks \citep{gyftodimos02} are another generalization of
Bayesian networks in which variables can be aggregations (or Cartesian
products) of other variables and a hierarchical tree
is used to represent relations between them. Both of these formalisms are very general and they are capable of representing
conditional independencies between variable groups. Therefore, our results may be applied to these models. However, these models are unnecessarily complicated for our analysis and thus we do not consider them further here.

Multiply sectioned Bayesian networks \citep{xiang93} model dependencies between overlapping variable groups. They are typically used to aid inference. 
They decompose a DAG into a hypertree where hypernodes are labelled by a subgraph and hyperlinks by separator sets. 
However, multiply sectioned Bayesian networks require variable groups to be overlapping and thus are not suitable for modelling dependencies between non-overlapping variable groups.

Module networks \citep{segal05} have been designed to handle large
data sets. The variables are partitioned into modules where the
variables in the same module share parents and parameters. Module
networks are particularly good for approximate density
estimation. However, their structural limitations make them unsuitable
for analysing conditional independencies between variable groups.

Huffman networks \citep{davis99} are Bayesian networks were nodes represent variable groups. They are designed to aid data compression and the variable groups are learned to enable efficient compression.

Burge and Lane \cite{burge06} have presented Bayesian networks for aggregation hierarchies which are related to hierarchical Bayesian networks. Groups of variables are aggregated by, for example, taking a maximum or mean and then networks are learned between the aggregated variables. From our point of view, the downside of this approach is that conditional independencies between aggregated variables do not necessarily correspond to conditional independencies between groups.

Entner and Hoyer \cite{entner12} have presented an algorithm for
finding causal structures among groups of continuous variables. Their
model works under the assumptions that variables are linearly related
and associated with non-Gaussian noise.

An earlier version of this paper \cite{parviainen16} appeared in the proceedings of the PGM 2016 conference. New contents of this paper include an analysis of the relationship between faithfulness and groupwise faithfulness (Theorems~\ref{thm:imply} and \ref{thm:imply2}), an alternative definition of causality for variable groups and an analysis of it (Definition~\ref{def:strong_causality} and Theorem~\ref{thm:strong_causality}), a new algorithm for learning group DAGs (Section~\ref{sec:simultaneous}), and more thorough experiments (Section~\ref{sec:experiments}).

\section{Preliminaries}

\subsection{Conditional Independencies}

Two random variables $x$ and $y$ are {\em conditionally independent}
given a set $S$ of random variables if $P(x, y| S) = P(x|S)
P(y|S)$. If the set $S$ is empty, variables $x$ and $y$ are marginally
independent. 
We use  $x\ci y|S$ to denote that $x$ and $y$ are conditionally independent  given $S$. 

Conditional independence can be generalized to sets of random variables. Two sets of random variables $X$ and $Y$ are conditionally independent given a set $S$ of random variables if $P(X, Y|S) = P(X|S) P(Y|S)$.

\subsection{Bayesian Networks}

A {\em Bayesian network} is a representation of a joint distribution of random variables. A Bayesian network consists of two parts: a structure and parameters. The structure of a Bayesian network is a directed acyclic graph (DAG) which expresses the conditional independencies and the parameters determine the conditional distributions.

Formally, a DAG is a pair $G = (N, A)$ where $N$ is the node set and $A$ is the arc set.  If there is an arc from $u$ to $v$, that is, $uv\in A$ then we say that $u$ is a {\em parent} of $v$ and $v$ is a {\em child} of $u$. The set of parents of $v$ in $A$ is denoted by $A_v$. Nodes $v$ and $u$ are  said to be {\em spouses} of each other if they have a common child and there is no arc between $v$ and $u$. Further, if there is a directed path from $u$ to $v$ we say that $u$ is an {\em ancestor} of $v$ and $v$ is a {\em descendant} of $u$.  The cardinality of $N$ is denoted by $n$. When there is no ambiguity on the node set $N$, we identify a DAG by its arc set $A$.

Each node in a Bayesian network is associated with a conditional probability distribution of the node given its parents. The conditional probability distribution of the  node is specified by the parameters. A DAG represents a joint probability distribution over a set of random variables if the joint distribution satisfies the {\em local Markov condition}, that is, every node is conditionally independent of its non-descendants given its parents. Then the joint distribution over a node set $N$ can be written as $P(N) = \prod_{v\in N} P(v| A_v)$ where the conditional probabilities for node $v$ are specified by the parameters $\theta_v$. We denote the set of all local parameters by $\Theta$. Finally, we define a Bayesian network to be a pair $(G, \Theta)$.

The conditional independencies implied by a DAG can be extracted using a d-separation criterion. 
The {\em skeleton} of a DAG $A$ is an undirected graph that is obtained by replacing all directed arcs $uv \in A$ with undirected edges between $u$ and $v$. A {\em path} in a DAG is a cycle-free sequence of edges in the corresponding skeleton. A node $v$ is a {\em head-to-head node} along a path if there are two consecutive arcs $uv$ and $wv$ on that path. Nodes $v$ and $u$ are {\em d-connected} by nodes $Z$ along a path from $v$ to $u$ if every head-to-head node along the path is in $Z$ or has a descendant in $Z$ and none of the other nodes along the path is in $Z$. Nodes $v$ and $u$ are {\em d-separated} by nodes $Z$ if they are not d-connected by $Z$ along any path from $v$ to $u$. 

Nodes $s$, $t$, and $u$ form a {\em v-structure} in a DAG if $s$ and $t$ are spouses and $u$ is their common child. 
Two DAGs are said to be {\em Markov equivalent} if they imply the same set of conditional independence statements. It can be shown that two DAGs are Markov equivalent if and only if they have the same skeleton and same v-structures \citep{verma90}. 

A distribution $p$ is said to be {\em faithful} to a DAG $A$ if $A$ and $p$ imply exactly the same set of conditional independencies. If $p$ is faithful to $A$ then $v$ and $u$ are conditionally independent given $Z$ in $p$ if and only if $v$ and $u$ are d-separated by $Z$ in
$A$. This generalizes 
to variable sets. That is, if $p$ is faithful to $A$  then variable sets $T$ and $U$ are conditionally independent given $Z$ in $p$ if and only if $t$ and $u$ are d-separated by $Z$ in $A$ for all $t\in T$ and $u\in U$.

\section{Groupwise Independencies}

In this section we introduce a new assumption, groupwise faithfulness,
that is necessary for principled learning of DAGs for variable groups. We will also
show that groupwise conditional independencies have a limited role in learning causal relations
between groups.

\subsection{Groupwise Faithfulness} \label{sec:gf}

First, let us introduce some terminology. Recall that $N$ is our node set. Let $W = \{W_1, \ldots, W_k\}$ be a collection of nonempty sets where $W_i\subseteq N\, \forall i$, and $W$ forms a partition of $N$. We call $W$ a {\em grouping}. We call a DAG on $N$ a {\em variable DAG} and a DAG on $W$ a {\em group  DAG}; Note that the nodes of the group DAG are subsets of $N$. We try to solve the following computational problem: We are given  a grouping $W$ and data $D$ from a distribution $p$ on variables $N$ that is faithful to a variable DAG $G$. The task is to learn a group DAG $H$ on $W$ such that for all $W_i, W_j\in W$ and $S = \cup_l T_l$, with $T = \{T_1, \ldots, T_k\}\subseteq W\setminus \{W_i, W_j\}$, it holds that $W_i$  and $W_j$ are d-separated by $S$ in $H$ if and only if   $W_i\ci W_j | S$ in $p$.

It is well-known that DAGs are not closed under marginalization. That is, even though the data-generating distribution is faithful to a DAG on a node set $N$, it is possible that the conditional independencies on some subset of $N$ are not exactly representable by any DAG. We note that DAGs are not closed under aggregation, either. By aggregation we mean representing conditional independencies among groups using a group DAG. We show that by presenting an example. Consider a distribution that is faithful to the DAG in Figure~\ref{fig:group_faithful1}(a). We want to express conditional independencies between groups $V_1$, $V_2$, and $V_3$. By inferring conditional independencies from the variable DAG, we get that $V_1\ci V_2$ and $V_1\ci V_2 | V_3$. There does not exist a DAG that expresses this set of conditional independencies exactly. 

\begin{figure}[hbt!]
\begin{center}
\hspace{-0.3in}
\begin{tabular}{ccc}
\scalebox{0.8}{\begin{tikzpicture}
\tikzstyle{main}=[circle, minimum size = 10mm, thick, draw =black!80, node distance = 16mm]
\tikzstyle{connect}=[-latex, thick]
\tikzstyle{box}=[rectangle, draw=black!100]

\node[main, fill = white!100] (x1) [label=center:$x_1$] { };
\node[main, fill = white!100] (x2) [right=of x1, label=center:$x_2$] { };
\node[main, fill = white!100] (x3) [below=of x1, label=center:$x_3$] { };
\node[main, fill = white!100] (x4) [below=of x2, label=center:$x_4$] { };
\path (x1) edge [connect] (x3)
	(x2) edge [connect] (x4)
	(x3) edge[connect] (x4);
\node[rectangle, inner sep=0mm, fit= (x3) (x4),label=below right:$V_3$, xshift=12mm, yshift=1mm] {};
\node[rectangle, inner sep=4.4mm,draw=black!100, fit= (x3) (x4)] {};
\node[rectangle, inner sep=0mm, fit= (x1),label=below right:$V_1$, xshift=-1mm, yshift=1mm] {};
\node[rectangle, inner sep=4.4mm,draw=black!100, fit= (x1)] {};
\node[rectangle, inner sep=0mm, fit= (x2),label=below right:$V_2$, xshift=-1mm, yshift=1mm] {};
\node[rectangle, inner sep=4.4mm,draw=black!100, fit= (x2)] {};
\end{tikzpicture}} &
\scalebox{0.8}{\begin{tikzpicture}
\tikzstyle{main}=[circle, minimum size = 10mm, thick, draw =black!80, node distance = 16mm]
\tikzstyle{connect}=[-latex, thick]
\tikzstyle{box}=[rectangle, draw=black!100]

\node[main] (x1) at (0,0) {$x_1$};
\node[main]  (x2) at (2,0) {$x_2$};
\node[main]  (x3) at (4,0) {$x_3$};
\node[main]  (x4) at (1,-2) {$x_4$};
\node[main]  (x5) at (3,-2) {$x_5$};

\path (x2) edge[connect] (x1)
	edge[connect] (x5)
	(x4) edge[connect] (x1)
	(x3) edge[connect] (x5);
	
\node[rectangle, inner sep=0mm, fit= (x5) (x4),label=below right:$V_3$, xshift=9mm, yshift=1mm] {};
\node[rectangle, inner sep=4.4mm,draw=black!100, fit= (x5) (x4)] {};
\node[rectangle, inner sep=0mm, fit= (x1) (x2),label=below right:$V_1$, xshift=9mm, yshift=1mm] {};
\node[rectangle, inner sep=4.4mm,draw=black!100, fit= (x1) (x2)] {};
\node[rectangle, inner sep=0mm, fit= (x3),label=below right:$V_2$, xshift=-1mm, yshift=1mm] {};
\node[rectangle, inner sep=4.4mm,draw=black!100, fit= (x3)] {};
\end{tikzpicture}}
&
\scalebox{0.8}{\begin{tikzpicture}
\tikzstyle{main}=[minimum size = 10mm, thick, draw =black!80, node distance = 16mm]
\tikzstyle{connect}=[-latex, thick]

\node[main] (v1) at (0,0) {$V_1$};
\node[main] (v2) at (3,0) {$V_2$};
\node[main] (v3) at (1.5,-1.5) {$V_3$};

\path (v1) edge[connect] (v3)
	(v2) edge[connect] (v3);
\end{tikzpicture}} 
\\
(a) & (b) & (c)
\end{tabular}
\caption{(a) A variable DAG where conditional independencies among groups $V_1$, $V_2$, and $V_3$ cannot be expressed exactly using any DAG. (b) A causal variable DAG where conditional independencies among groups $V_1$, $V_2$, and $V_3$ lead to a group DAG in which v-structures cannot be interpreted causally. (c) A group DAG corresponding to causal variable DAG in (b).
\label{fig:group_faithful1}
}
\end{center}
\end{figure}
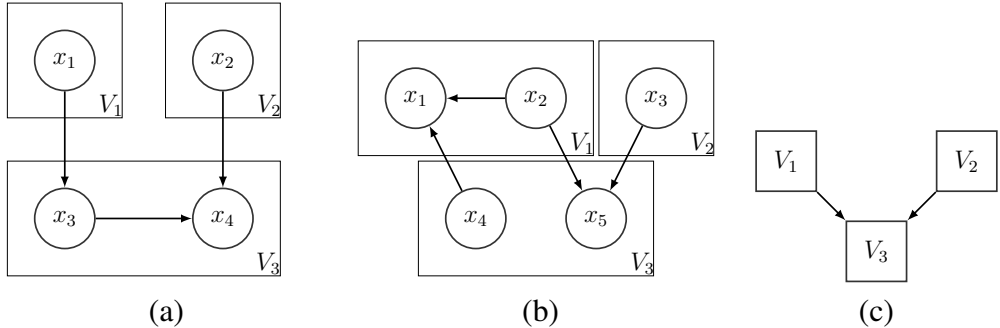

To avoid cases where conditional independencies are not representable by any group DAG, we introduce a new assumption: groupwise faithfulness. Formally, we define groupwise faithfulness as follows.

\begin{assumption}[Groupwise faithfulness] \label{assump:gf}
A distribution $p$ is groupwise faithful to a group DAG $H$ given a
grouping $W$, if $H$ implies the exactly same set of conditional
independencies as $p$ over the groups $W$. 
\end{assumption} 

Note that this assumption is analogous with the faithfulness assumption in the sense that in both cases there exists a DAG that expresses exactly the independencies in the distribution.

Sometimes it is convenient to investigate whether conditional independencies implied by a variable DAG given a grouping are equivalent to the conditional independencies implied by a group DAG. We will use this notion later in this section when we investigate the strength of the groupwise faithfulness assumption.

\begin{assumption}[Groupwise Markov equivalence] \label{assump:sgf}
A variable DAG $G$ is groupwise Markov equivalent to a group DAG $H$ given a grouping $W$, if $H$ implies the exactly same set of conditional independencies as $G$ over groups $W$. 
\end{assumption}

We note that if a distribution $p$ is faithful to a DAG $G$, and $G$ is groupwise Markov equivalent to a DAG $H$ given a grouping $W$, then $p$ is groupwise faithful to $H$ given $W$. This shows that faithfulness and groupwise Markov equivalence together imply groupwise faithfulness. However, neither faithfulness nor groupwise Markov equivalence alone is necessary or sufficient for groupwise faithfulness. 

To see this, let us consider the following examples. First, to see that
faithfulness is not sufficient for groupwise faithfulness, assume that
we have a distribution that is faithful to the DAG in
Figure~\ref{fig:group_faithful1}(a). Given groups $V_1$, $V_2$, and $V_3$,
the distribution is groupwise unfaithful. Second, consider a
distribution over the variable set $x_1$, $x_2$, $x_3$, $x_4$, and
$x_5$. Let us assume that the groups are $V_1 = \{x_1, x_2\}$, $V_2 =
\{x_3\}$, and $V_3 = \{x_4, x_5\}$ and the Bayesian network factorizes
according to the variable DAG in Figure~\ref{fig:group_faithful1}(b). Now, it is possible to construct a distribution such that the local conditional distribution at node $x_1$ is
exclusive or (XOR), and thus the variable DAG is unfaithful. If the other local conditional distributions do not introduce any additional independencies then the distribution is groupwise faithful. This shows
that faithfulness is not necessary for groupwise faithfulness. Next,
let us consider the same structure but let us assume that both $x_1$
and $x_5$ are associated with XOR distributions. In this case the
variable DAG is groupwise Markov equivalent to the group DAG but the distribution is not groupwise
faithful which shows that groupwise Markov equivalence is not
sufficient for groupwise faithfulness. Finally, consider the variable
DAG and the grouping in Figure~\ref{fig:group_faithful1}(a). This variable DAG is
not groupwise Markov equivalent to the group DAG given the grouping. However, if the distribution is
unfaithful to the DAG and the variables $x_1$ and $x_3$ are independent then
the distribution is groupwise faithful. This shows that 
groupwise Markov equivalence is not necessary for groupwise faithfulness. 
As neither faithfulness nor groupwise Markov equivalence is sufficient or necessary for groupwise faithfulness, it follows that groupwise faithfulness implies neither faithfulness nor groupwise Markov equivalence. 

As neither faithfulness nor structural groupwise faithfulness is sufficient or necessary for groupwise faithfulness, it follows that groupwise faithfulness implies neither faithfulness or structural groupwise faithfulness. 

We  have also studied whether groupwise faithfulness together with certain kinds of group DAGs or groupings imply faithfulness. It turns out that groupwise faithfulness implies faithfulness only when the maximum group size is one and in some special cases when the maximum group size is two  as stated in the theorems below; the proofs of the theorems are found in \ref{app:proofs}.

\begin{restatable}{theorem}{thmimply} \label{thm:imply}
Let $H$ be a group DAG on a grouping $W$. Then every distribution $p$ on $\cup W_i$ that is groupwise faithful to $H$ given $W$ is faithful to some variable DAG on $\cup W_i$ if  $\max_{W_i\in W} |W_i| = 1$ or $\max_{W_i\in W} |W_i| = 2$ and no group of size 2 has neighbors in $H$.
\end{restatable}

\begin{restatable}{theorem}{thmimplyb} \label{thm:imply2}
Let $H$ be a group DAG on a grouping $W$. If $\max_{W_i\in W}
|W_i|\geq 3$, or $\max_{W_i\in W} |W_i| = 2$ and two groups of size $2$
are adjacent in the group DAG, then there exists a distribution $p$
such that $p$ implies the same set of groupwise conditional
independencies as $H$ on $W$ and $p$ is not faithful to any DAG.
\end{restatable}

Note that there is a ``gap'' between the above theorems; we do not know whether or not groupwise faithfulness implies faithfulness when the maximum group size is $2$ and the groups of size $2$ have neighbors of size $1$.

Next, we will explore how strong the groupwise faithfulness assumption is. That is, how likely we are to encounter groupwise faithful distributions. To this end, we consider distributions that are faithful to variable DAGs.
The joint space of DAGs and groupings is too large to be enumerated and we are not aware of any formula for assessing the number of groupwise unfaithful networks. Therefore, we analyze the prevalence of groupwise faithfulness by an empirical evaluation using simulations.

In simulations, a key question is how to check groupwise faithfulness. That is, given a variable DAG and a grouping, how to check whether the conditional independencies entailed by the variable DAG over groups can be represented exactly using a group DAG. Because the data-generating distribution is faithful to a variable DAG, we check whether the variable DAG over groups is groupwise Markov equivalent to some group DAG. This can be done by first using the PC algorithm \citep{spirtes00} to construct a group DAG; here we use d-separation in the variable DAG as our independence test. Once the group DAG has been constructed we can check that the set of conditional independencies entailed by the group DAG is exactly the set of groupwise conditional independencies implied by the variable DAG and the grouping. The PC algorithm is sound and complete so if there exists a DAG that implies exactly  the set of given conditional independencies, then the PC algorithm returns (the equivalence class of) that DAG. Thus, the conditional independencies match if and only if the variable DAG and the grouping are groupwise Markov equivalent to a group DAG.

We used the Erd\H{o}s-R\'enyi model \citep{erdos59, gilbert59} to generate random DAGs. A DAG
from model $G(n, p)$ has $n$ nodes and each arc is included with
probability $p$ independently of all other arcs; to get an acyclic
directed graph, we fix the order of nodes.  We generated random DAGs
with $n =20$ by varying the parameter $p$ from $0.1$ to $0.9$. We also
generated random groupings where group size was fixed to 2, 3, 4, or 5
(20 is not divisible by 3, so in this case one group is smaller than
the others). For each value of $p$, we generated 100 random graphs. Then,
we generated 10 groupings for each graph for each group size and
counted the proportion of groupwise faithful DAG-grouping pairs. The
results are shown in Figure~\ref{fig:gf_simu}. It can be seen that
groupwise unfaithfulness is probable with sparse graphs and small
group sizes. One should, however, note that the simulation results are
for random graphs and groupings, and real life graphs and groupings may
or may not follow this pattern.

\begin{figure}[hbt!]
\begin{center}
\hspace{-0.3in}
\includegraphics[width=0.6\linewidth]{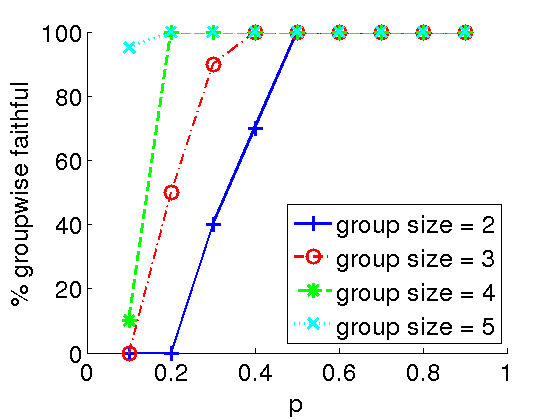}  
\caption{Proportion of DAG-grouping pairs that are groupwise faithful in random graphs of 20 nodes. Parameter $p$ is the probability that an arc is present.
\label{fig:gf_simu}
}
\end{center}
\end{figure}

\subsection{Causal Interpretation} \label{sec:causality}

Probabilistic causation between variables is typically defined to concern predicting effects of interventions. This means that an external manipulator intervenes the system and forces certain variables to take certain values. In our context, we say that a group $V$ {\em causes} group $U$ if intervening on $V$ affects the joint distribution of $U$.

While the above definition does not require the distribution to be of any particular form, we concentrate on our analysis on distributions that can be represented using causal DAGs.
A DAG is called {\em causal} if it satisfies the {\em causal Markov condition}, that is, all variables are conditionally independent of their non-effects given their direct causes. Assuming faithfulness and causal sufficiency (if any pair of observed variables has a common cause then it is observed), it is possible to identify causal effects  using the {\em do}-operator  \citep{pearl00}. The {\em do}-operator $do(v = v_1)$ sets the value of the variable $v$ to be $v_1$. The probability $P(u | do(v = v_1))$ is the conditional probability distribution of $u$ given that the variable $v$ has been forced to take value $v_1$. In other words, one takes the original joint distribution, removes all  arcs that head to $v$ and sets $v=v_1$; then one computes the probability $P(u | v=v_1)$ in the new distribution. 
We define a cause using the so-called operational criterion for causality \citep{aliferis10a}, that is, we say that a variable $v$ is a {\em cause} (direct or indirect) of a variable $u$ if and only if $P(u | do(v = v_1)) \neq P(u| do(v = v_2))$ for some values $v_1$ and $v_2$. 
A straightforward generalization leads to the following definition of causality for variable groups.

\sloppy
\begin{definition}[Group causality]
Assuming that $P$ is a causal Bayesian network and given variable groups $V$ and $U$, $V$ is a cause of $U$ if $P(U | do(V = V_1)) \neq P(U | do(V = V_2))$ for some instantiations $V_1$ and $V_2$ of values of $V$.
\end{definition}

Note that the above definition allows causal cycles between groups. To see this, consider a causal DAG on $\{v_1, v_2, v_3, v_4\}$ which has arcs $v_1 v_3$ and $v_4 v_2$. If there are two groups $W_1 = \{v_1,  v_2\}$ and $W_2 = \{v_3,  v_4\}$ then $W_1$ is a cause of $W_2$ (because there is a causal arc $v_1 v_3$) and $W_2$ is a cause of $W_1$ (because of a causal arc $v_4 v_2$).

In the above, we assumed that the variable DAG is causal. An alternative scenario is to assume both the group DAG and the variable DAG are causal.
This results in the following, stronger definition of causality which does not allow causal cycles between groups.

\begin{definition}[Strong group causality] \label{def:strong_causality}
Assuming that $P$ is a causal Bayesian network and given variable groups $V$ and $U$, $V$ is a strong cause of $U$ if $V$ is a cause of $U$ and $U$ is not a cause of $V$ in $P$.
\end{definition}

Next, we will study to what extent causality between variable groups can be detected from observational data using only conditional independencies among groups. We assume that the data come from a distribution that is faithful to a causal variable DAG. Further, we assume that we have no access to the raw data but only to an oracle that conducts conditional independence tests. 
Formally, we assume that we have access to an oracle $\mathcal{O}_G$ that answers queries $W_i \ci W_j | S$, where $W_i, W_j \in W$ and $S= \cup_l T_l$ with $T = \{T_1, \ldots T_m\}\subseteq W\setminus \{W_i, W_j\}$. Note that in the standard scenario with conditional independencies between variables, we have an oracle $\mathcal{O}_V$ that answers queries $X\ci Y | Z$, where $X, Y\in N $ and $Z\subseteq N\setminus \{X, Y\}$; If $\max_i |W_i| > 1$ then the oracle $\mathcal{O}_V$ is strictly more powerful than $\mathcal{O}_G$.

It is well-known that, under standard assumptions, a causal variable DAG can be learned up to the Markov equivalence class. A Markov equivalence class can be represented by a completed partial DAG (CPDAG) where we have both directed and undirected edges. Directed edges or arcs are the edges that point to the same direction in every member of the equivalence class whereas undirected edges express cases where the edge is not directed to the same direction in all members of the equivalence class. If there is a directed path from a variable $v$ to a variable $u$ in the CPDAG then $v$ is a cause of $u$. In other words, existence of such a path is a sufficient condition for causality. However, it is not a necessary condition and it is possible that $v$ is a cause of $u$ even when there is no directed path from $v$ to $u$ in the CPDAG.

Next, we consider causality in the group context. 
Manipulating an ancestor of a node affects its distribution and thus the ancestor is a cause of its descendant. It is easy to see that given a causal variable DAG $G$, a group $W_i$ is a group cause of a group $W_j$ if and only if there is at least one directed path from $W_i$ to $W_j$ in $G$, that is, there exists $v\in W_i$ and $u\in W_j$ such that there is a directed path from $v$ to $u$. 
It is clear from the above that a sufficient condition for a group $W_i$ to be a  group cause of  a group $W_j$ is that there is at least one directed path from $W_i$ to $W_j$ in the CPDAG.

Standard constraint-based algorithms for causal learning start by constructing a skeleton and then directing arcs based on a set of rules. So let us take a look on these rules in the group context. The first rule is to direct v-structures.
The following theorem shows that arcs that are part of a v-structure in a group DAG imply group causality.

\begin{theorem} \label{thm:weak_v}
Let $N$ be a node set and $W$ a grouping on $N$. Let  $p$ be a distribution that is groupwise faithful to some group DAG $H$ given the grouping $W$. If there exist groups
$W_i,W_j, W_k\in W$  such that (i)  $W_i\ci W_k | S$ for some  $S\subseteq W\setminus\{W_i, W_j, W_k\}$ and (ii) $W_i\not\ci W_k | (\cup_l T_l) \cup W_j$ for all $T = \{T_1, \ldots, T_m\}\subseteq W\setminus\{W_i, W_j, W_k\}$ then $W_i$ is a group cause of $W_j$.
\end{theorem}
\begin{proof}
It is sufficient to show that there exists a pair $w_i\in W_i$ and $w_j\in W_j$ such that $w_i$ is an ancestor of $w_j$ in the variable DAG.

Due to (i), all paths that go from $W_i$ to $W_k$  without visiting $S$ must have a head-to-head node. Due to (ii) there has to exist at least one path between $W_i$ and $W_k$ such that there are no non-head-to-head nodes in $W\setminus\{W_i, W_k\}$ and all head-to-head nodes are unblocked by $W_j$; let us denote one such a path by $R$. Without loss of generality, we can assume that all nodes in $R$ except the endpoints are in $W\setminus \{W_i, W_k\}$. Let $s, t, u\in N$ be three consecutive nodes in path $R$ such that there are edges $st$ and $ut$. Nodes $s$ and $u$ cannot be head-to-head nodes along $R$ and therefore $s, u\in W_i\cup W_k$. Node $t$ is a head-to-head node and therefore either $t\in W_j$ or $t$ has a descendant in $W_j$. In both cases there is a directed path from both $s$ and $u$ to the set $W_j$. The path $R$ has one end-point in $W_i$ and another in $W_k$. Thus, there is a directed path from $W_i$ to $W_j$ in the variable DAG.
\end{proof}

Note that the proof of the previous theorem implies that there is a v-structure $W_i \rightarrow W_j \leftarrow W_k$ in the group DAG only if there exists $w_i\in W_i$, $w_j\in W_j$, and $w_k\in W_k$ such that there exists a v-structure $w_i \rightarrow w_j \leftarrow w_k$ in the variable DAG.

After v-structures have been directed, one can direct the rest of the edges that point to the same direction in every DAG of the Markov equivalence class using four local rules often referred to as the Meek rules \citep{meek95b}. The rules are \citep{pearl00}:
\begin{enumerate}
\item[R1:] Orient $v - s$ into $v\rightarrow s$ if there is an arrow $u\rightarrow v$ such that $u$ and $s$ are nonadjacent. 
\item[R2:] Orient $u-v$ into $u\rightarrow v$ if there is a chain $u\rightarrow s\rightarrow v$.
\item[R3:] Orient $u-v$ into $u\rightarrow v$ if there are two chains $u - s \rightarrow v$ and $u - t \rightarrow v$ such that $s$ and $t$ are nonadjacent.
\item[R4:] Orient $u- v$ into $u\rightarrow v$ if there are two chains $u - s\rightarrow t$ and $s \rightarrow t \rightarrow v$ such that $s$ and $v$ are nonadjacent and $u$ and $t$ are adjacent.
\end{enumerate}

We would like to generalize these rules for variable groups. However, these rules are not sufficient  to infer group causality if one does have access only to the groupwise conditional independencies (and to nothing else). To see this, consider a group DAG $H = (W, E)$ where $W = \{S, T, U, V\}$ and $E = \{SU, TU, UV\}$ shown in Figure~\ref{fig:causal1}(a). Now, Theorem~\ref{thm:weak_v} says that $S$ and $T$ are causes of $U$. The rule R1 suggest that we could claim that $U$ is a cause of $V$. However, we can construct a causal variable DAG $G = (N, F)$ with $N = \{s_1, s_2, t_1, t_2, u_1, u_2, u_3, v_1, v_2\}$ and $F = \{u_1s_1, v_1u_1, s_2u_2, t_1u_2, v_2u_3, u_3t_2 \}$ and $S = \{s_1, s_2\}$, $T = \{t_1, t_2\}$, $U = \{u_1, u_2, u_3\}$, and $V = \{v_1, v_2\}$; see Figure~\ref{fig:causal1}(b). Clearly, $G$ implies the same conditional independencies on $W$ as does $H$ and there is no directed path from $U$ to $V$ in $G$. Thus, $U$ is not a cause of $V$ in $G$.

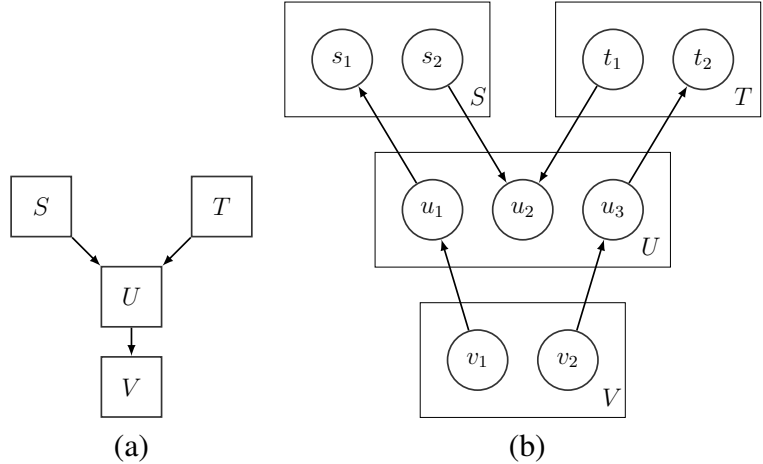
\begin{figure}[hbt!]
\begin{center}
\hspace{-0.3in}
\begin{tabular}{cc}
\scalebox{0.8}{\begin{tikzpicture}
\tikzstyle{main}=[minimum size = 10mm, thick, draw =black!80, node distance = 16mm]
\tikzstyle{connect}=[-latex, thick]

\node[main] (S) at (0,0) {$S$};
\node[main] (T) at (3,0) {$T$};
\node[main] (U) at (1.5,-1.5) {$U$};
\node[main] (V) at (1.5,-3) {$V$};

\path (S) edge[connect] (U)
	(T) edge[connect] (U)
	(U) edge[connect] (V);
\end{tikzpicture}} 
&
\scalebox{0.8}{\begin{tikzpicture}
\tikzstyle{main}=[circle, minimum size = 10mm, thick, draw =black!80, node distance = 16mm]
\tikzstyle{connect}=[-latex, thick]
\tikzstyle{box}=[rectangle, draw=black!100]

\node[main] (s1) at (0,0) {$s_1$};
\node[main]  (s2) at (1.5,0) {$s_2$};
\node[main]  (t1) at (4.5,0) {$t_1$};
\node[main] (t2) at (6,0) {$t_2$};
\node[main]  (u1) at (1.5,-2.5) {$u_1$};
\node[main]  (u2) at (3,-2.5) {$u_2$};
\node[main]  (u3) at (4.5,-2.5) {$u_3$};
\node[main]  (v1) at (2.25,-5) {$v_1$};
\node[main]  (v2) at (3.75,-5) {$v_2$};

\path (s2) edge[connect] (u2)
	(t1) edge[connect] (u2)
	(v1) edge[connect] (u1)
	(u1) edge[connect] (s1)
	(v2) edge[connect] (u3)
	(u3) edge[connect] (t2);
	
\node[rectangle, inner sep=0mm, fit= (s1)(s2),label=below right:$S$, xshift=7mm, yshift=2mm] {};
\node[rectangle, inner sep=4.4mm,draw=black!100, fit= (s1) (s2)] {};
\node[rectangle, inner sep=0mm, fit= (t1) (t2),label=below right:$T$, xshift=6mm, yshift=2mm] {};
\node[rectangle, inner sep=4.4mm,draw=black!100, fit= (t1) (t2)] {};
\node[rectangle, inner sep=0mm, fit= (u1)(u2)(u3),label=below right:$U$, xshift=13mm, yshift=2mm] {};
\node[rectangle, inner sep=4.4mm,draw=black!100, fit= (u1)(u2)(u3)] {};
\node[rectangle, inner sep=0mm, fit= (v1) (v2),label=below right:$V$, xshift=6.5mm, yshift=2mm] {};
\node[rectangle, inner sep=4.4mm,draw=black!100, fit= (v1) (v2)] {};
\end{tikzpicture}}\\
(a) & (b) \\
\end{tabular}
\caption{(a) A group DAG and (b) a causal variable DAG that implies the same groupwise independencies.
\label{fig:causal1}
}
\end{center}
\end{figure}

The above observation implies that the Meek rules cannot be used to infer causality in group DAGs. However, it is not known whether there are some special conditions under which the Meek rules would apply in this context. Note that the above applies only when the conditional independencies between individual variables are not known; when the variable DAG is known, this information can be used to help to infer more causal relations. 

Let us analyze detecting strong group causality.
The theorem below shows that none of the arcs in the group DAG imply strong group causality if minimum group size is at least 2. 

\begin{theorem} \label{thm:strong_causality}
We are given a node set $N$, a grouping $W$, and a group DAG $H$. If $|W_i| > 1$ for all $i$ then $W_j$ being  an ancestor of $W_k$ in $H$  does not imply that $W_j$ is a strong group cause of $W_k$.
\end{theorem}
\begin{proof}
By the definition of strong group cause, if $W_j$ is a strong group cause of $W_k$ then $W_k$ is not a group cause of $W_j$. Thus, to prove the theorem, it is sufficient to show that for any group DAG $H$  and a grouping $W$ with $|W_i| >1$ for all $i$ there exists a causal variable DAG in which $W_k$ is a group cause of $W_j$.
In other words, it is sufficient to show that for any group DAG $H$ on $W$, where $W_j$ is an ancestor of $W_k$,  it is possible to construct a causal variable DAG $G$ on $N$ such that $G$ given $W$ implies the same conditional independencies as $H$, and there exists a pair $w_k\in W_k$ and $w_j\in W_j$ such that there is a directed path from $w_k$ to $w_j$ in $G$. 

Next, we will show how to construct such a variable DAG. 
Let $H$ be the group DAG on $W$ expressing groupwise conditional independencies. Without loss of generality, we can choose two distinct nodes $w_i^1$ and $w_i^2$ from each group $W_i$. Now consider the following causal variable DAG $G'$ on $N$. We start by setting $G'$ to be an empty DAG. Then, we add edges from $w_i^1$ to $w_l^1$ for all $i$ and $l$ such that there is an edge from $W_i$ to $W_l$ in $H$. Finally, we select a directed path $R$ from $W_j$ to $W_k$ and add an edge from $w_l^2$ to $w_i^2$ to $G'$ if there is an edge from $W_i$ to $W_l$ on $R$; note that $W_j$ is an ancestor of $W_k$ so there exists at least one directed path from $W_j$ to $W_k$.  

It remains to show that the above construction has the desired properties, that is, $G'$ given $W$ implies the same conditional independencies as $H$, and there exists a pair $w_k\in W_k$ and $w_j\in W_j$ such that there is a directed path from $w_k$ to $w_j$ in $G'$.
It is clear that the induced graph on $w_i^1$-variables imply exactly the same groupwise conditional independencies as $H$. Furthermore, there is a path from $W_j$ to $W_k$ in $H$ and the $w_i^2$-variables encode the same path in reverse, and do not express any dependencies that are not already implied by the $w_i^1$-variables; in other words, if $w_i^2$ and $w_l^2$ are d-connected given $S\subseteq W\setminus \{W_i, W_l\}$ in $G'$ then $W_i$ and $W_l$ are d-connected given $S$ in $H$. Therefore, $H$ implies exactly the same conditional independencies on $W$ as $G'$ given $W$. Furthermore, due to the existence of a path from $w_k^2 \in W_k$ to $w_j^2\in W_j$ in the causal variable DAG $G'$, $W_j$ is not a strong group cause of $W_k$. This is sufficient to show that one cannot infer strong group causality using only groupwise conditional independencies.
\end{proof}

\section{Learning group DAGs}

Next, we will introduce three approaches for learning group DAGs.

\subsection{Direct Learning} \label{sec:direct}

The most straightforward approach is to learn a group DAG directly,
that is, either using conditional independencies or local scores on a
grouping $W$. In other words, we can consider each group as a variable. Assuming that the variables are discrete, the possible states of the new variable $w_i$, corresponding to the group $W_i$, are the Cartesian product of the states of the variables in $W_i$. Now there is a bijective mapping between joint configurations of variables in $W_i$ and states of $w_i$. Thus $W_i \ci W_j | S_1$ if and only if $w_i \ci w_j | S_2$ where $W_l\subseteq S_1$ if and only if $w_l\in S_2$. 
This leads to a simple procedure described in Algorithm~\ref{alg:dl}.  
\begin{algorithm} [!hbt]
\caption{\sc FindGroupDAG1} \label{alg:dl}
\begin{algorithmic}[1]
\REQUIRE Data $D$ on a node set $N$, a grouping $W$ on $N$.
 \ENSURE Group DAG $G$
 \STATE Convert variables $x_i \in N$ into new variables $y_j$ on $W$ such that $y_j = \times_{x_i \in W_j} x_i$ .
 \STATE Learn a DAG $G$ on the new variables on $W$ using procedure \vdag.
 \RETURN $G$
 \end{algorithmic}
\end{algorithm}

The procedure \vdag in the second step is an algorithm for finding a DAG; it can use either the constraint-based or score-based approach. In principle, \vdag can be any learning algorithm. However, if \vdag is an exact algorithm then we can prove some theoretical guarantees; see Theorems~\ref{thm:constraint-based-group} and \ref{thm:score-based-group} below. We will next prove the correctness of the algorithm for the constraint-based approach. First, we state a well-known lemma that is used in the proof.

\begin{lemma}[\cite{spirtes00}] \label{thm:pc}
Given data $D$ on variables $V$, if $V$ is causally sufficient, the data-generating distribution is faithful to a DAG $A$, and the sample size tends to infinity then the PC algorithm finds a DAG that is Markov equivalent to $A$.
\end{lemma}

\begin{theorem} \label{thm:constraint-based-group}
Let data $D$ be generated from a Bayesian network $(G, \Theta)$ which is groupwise faithful to a DAG $G'$ given a grouping $W$. If causal sufficiency holds, the sample size tends to infinity, and the procedure \vdag uses the PC algorithm then Algorithm~\ref{alg:dl} returns a structure $H$ that is Markov equivalent  to $G'$.
\end{theorem}
\begin{proof}
Let an assignment of values of variables in $W_j$ be denoted by $W_j = w$ and assignment of the state of $y_j$ be denoted by $y_j = y$.
By the definition of $y_j$, each value $y$ of $y_j$ corresponds to exactly one assignment $w$. Thus, for every $y$ there exists a $w$ such that $P(y_j = y | S) = P(W_j = w| S)$ for all $S\in 2^{W\setminus W_j}$.  Therefore, $y_i \ci y_j | S$ if and only if $W_i \ci W_j | S$.

Causal sufficiency and groupwise faithfulness guarantee that the data-generating distribution has a perfect map $G'$ on $W$. Thus, by causal sufficiency, groupwise faithfulness, infinite sample size, and Lemma~\ref{thm:pc}, Algorithm~\ref{alg:dl} returns a DAG $G$ that is equivalent to $G'$. 
\end{proof}

The same result can easily be extended to the score-based approach;
see Theorem~\ref{thm:score-based-group} below. We assume that the scoring criterion is consistent. To this end, we say that a distribution $p$ is {\em contained} in a DAG $G$  if there exist parameters $\Theta$ such as $(G, \Theta)$ represents $p$ exactly. We are given i.i.d. samples $D$ from some distribution $p$. A scoring criterion $S$ is said to be {\em consistent} if, when the sample size tends to infinity, (1) $S(G, D) > S(H, D)$ for all $G$ and $H$ such that $p$ is contained in $G$ but not in $H$ and (2) $S(G, D) > S(H, D)$ if $p$ is contained in both $G$ and $H$ and $G$ has less parameters that $H$; for a more formal treatment of consistency, see, e.g., \citep{slobodianik09}.
The proof is analogous
to the proof above; instead of Lemma~\ref{thm:pc} one simply uses
the fact (Proposition~8 in \citep{chickering02b}) that if $V$ is causally sufficient, the data-generating distribution is faithful to a DAG,  a consistent scoring criterion is used and the sample size tends to infinity, then exact score-based algorithms return a DAG that is equivalent to the data-generating DAG .

\begin{theorem}
Let data $D$ be generated from a Bayesian network $(G, \Theta)$ which
is groupwise faithful to a DAG $G'$ given the grouping $W$. If causal sufficiency holds, the sample size tends to infinity,  and
 the procedure \vdag uses an exact score-based algorithm with a
consistent scoring criterion then
Algorithm~\ref{alg:dl} returns a structure $H$ that is Markov
equivalent to $G'$.
\label{thm:score-based-group}
\end{theorem}

\subsection{Learning via Variable DAGs} \label{sec:var_net}

We note that a DAG over individual variables specifies also all the conditional independencies and dependencies between groups. Thus, it is possible to learn a group DAG by first learning a variable DAG and then inferring the group DAG. Algorithm~\ref{alg:vn} summarizes this approach.

\begin{algorithm}[!hbt] 
\caption{{\sc FindGroupDAG2}} \label{alg:vn}
\begin{algorithmic}[1]
\REQUIRE Data $D$ on a node set $N$, a grouping $W$ on $N$.
\ENSURE  Group DAG $G$
\STATE Learn a DAG $H$ on $N$ using procedure \vdag.
\STATE Learn a group DAG $G$ on $W$ using the PC algorithm and d-separation in $H$ as an independence test.
\RETURN $G$
\end{algorithmic}
\end{algorithm}

The procedure \vdag can again be either constraint-based or score-based. The following theorem shows the theoretical guarantees of the algorithm assuming that \vdag is exact.

\begin{theorem}
Let data $D$ be generated from a Bayesian network $(G, \Theta)$ which
is groupwise faithful to a DAG $G'$ given the grouping $W$. If causal sufficiency and faithfulness hold, the sample size tends to infinity,  and the procedure \vdag uses the PC
algorithm,  Algorithm~\ref{alg:vn} returns a structure $H$ that is
Markov equivalent to $G'$.
\end{theorem}
\begin{proof}
As causal sufficiency and faithfulness hold, there exists a variable DAG that is a
perfect map of the data-generating distribution, and because of infinite sample size and
Lemma~\ref{thm:pc}, the DAG $H$ is that perfect map. By
groupwise faithfulness, the conditional independencies implied by
$H$ given the grouping $W$, can be expressed exactly by a group DAG. Thus
by Lemma~\ref{thm:pc}, Algorithm~\ref{alg:vn} returns a DAG $G$ that
is Markov equivalent to $G'$.
\end{proof}

Again, the above result holds also for score-based methods as summarized below.
\begin{theorem}
Let data $D$ be generated from a Bayesian network $(G, \Theta)$ which
is groupwise faithful to a DAG $G'$ given grouping $W$. 
If causal sufficiency and faithfulness hold, the sample size tends to infinity, and  the procedure \vdag uses an exact
score-based algorithm with a consistent scoring criterion, then Algorithm~\ref{alg:vn}
returns a structure $H$ that is Markov equivalent to $G'$.
\end{theorem}

\subsection{Combined learning}  \label{sec:simultaneous}

The combined learning algorithm is based on the score-based approach and learns both the variable DAG and the group DAG simultaneously under an assumption that the topological orders of the variable DAG and the group DAG are compatible. This algorithm is a variant of the dynamic programming algorithm by Silander and Myllym\"aki \cite{silander06}. The pseudocode is shown in Algorithm~\ref{alg:causal}. For simplicity, we show only how to compute the score of the group DAG; the DAG can be constructed in the similar fashion as in Silander and Myllym\"aki \cite{silander06}, by keeping track of which parent sets contributed to the score.

The algorithm begins with computing local scores for node--parent set pairs and finding the highest scoring parent set from the subsets of a given set (Lines~1--4). Then the algorithm proceeds to find the highest scoring DAG for each subset of the groups using dynamic programming (Lines~6--14). For each subset, one variable group is going to be a sink, that is, it has no children in the particular subset. Assuming that $W_i$ is the sink of the set $T$, the algorithm computes score for node $W_i$ given that the parents of $W_i$ are chosen from $T\setminus W_i$. This is computed by finding the score of the best DAG for nodes in $W_i$ given that each node is allowed to take parents from $T$ (Lines~8--11). The parent set of $W_i$ is then the union of all groups in $W_j\in T\setminus W_i$ such that at least one of the variables in $W_j$ is a parent of at least one variable of $W_i$ in the DAG found on Line~11. The score of the best group DAG on $T$ given that $W_i$ is a sink is the sum of the score of the sink and the score of the best DAG for the rest of the nodes. To find an optimal group DAG on $T$, one loops over all possible choices of sink and chooses the one with the highest score (Line~13). Finally, the optimal group DAG for the whole grouping is returned (Line~15).

\begin{algorithm}[!hbt] 
\caption{{\sc FindGroupDAG3}} \label{alg:causal}
\begin{algorithmic}[1]
\REQUIRE Data $D$ on a node set $N$, a grouping $W$ on $N$, the maximum number of parents $c$.
\ENSURE  A group DAG $H$
\FORALL {$v\in N$ and $S\subseteq N\setminus\{v\}, |S|\leq c$ in the order of increasing cardinality of $S$}
\STATE Store the local score for $v$ and $S$ to $s[v, S]$ 
\STATE $bs[v, S] = \max_{U\subseteq S} s[v, S]$
\ENDFOR
\STATE $B[\emptyset] = 0$
\FORALL {$T\in 2^W$ in the order of increasing cardinality}
\FORALL {$W_i \in T$}
\FORALL {$v\in W_i$ and $U\subseteq W_i\setminus\{v\}$}
\STATE $bss[v, U] = \max_{R\subseteq U\cup (T \setminus W_i)} bs[v, R] $
\ENDFOR
\STATE $ss[W_i, T\setminus W_i]$ =  the score of a highest scoring variable DAG on members of $W_i$ given local scores $bss[v, U]$
\ENDFOR
\STATE $B[T] = \max_{W_i\in T} \Big( ss[W_i, T\setminus W_i] + B[T\setminus W_i] \Big) $
\ENDFOR
\RETURN $B[N]$
\end{algorithmic}
\end{algorithm}

Let us analyze the time requirement of the algorithm. Recall that we have $n$ variables and $k$ groups. Let us use $n_{\max} = \max_{i} |W_i|$ to denote the size of the largest group. The first loop (Line~1) is executed $O(n^{c + 1})$ times. Finding the highest-scoring subset can be done using an additional $O(n)$ time \cite{silander06}. Thus, the first loop takes a total $O(n^{c + 2})$ time. Let us consider the loop starting at Line~6. The outmost loop is executed $2^k$ times and the second loop at most $n_{\max}$ times. The loop on Line~8 is executed at most $n_{\max} 2^{n_{\max} - 1}$ times. The computation of Line~9 can be done re-using values computed in previous steps by a straightforward adaptation of methods presented by Silander and Myllym\"aki  \cite{silander06}, with an additional cost of $O(n_{\max})$. The computation of Line~11 uses the standard Silander-Myllym\"aki algorithm and is done in $O(n_{\max} 2^{n_{\max}})$ time. This yields a total time requirement $O(n^{c + 2} + 2^{k + n_{\max}} n_{\max}^2 )$.

Note that finding a highest-scoring variable DAG using dynamic programming takes $O(n^2 2^n)$ time, so if the number of the groups and the sizes of the groups are approximately equal, the combined learning algorithm is considerably faster.

The following theorem provides theoretical guarantees for the algorithm.

\begin{theorem}
Let data $D$ be generated from a Bayesian network $(G, \Theta)$ which
is groupwise faithful to a DAG $G'$ given a grouping $W$ and whose topological order is compatible
with $G'$. If causal sufficiency and faithfulness hold, the sample size tends to infinity, and  the
procedure \vdag uses an exact score-based algorithm with a
consistent scoring criterion
then Algorithm~\ref{alg:causal} returns a structure $H$ that is Markov
equivalent to $G'$.
\end{theorem}
\begin{proof}
Given causal sufficiency, faithfulness, infinite sample size, and a consistent scoring criterion, $G$ is a highest scoring variable network. Because $G$ and $G'$ are compatible, all parents of members of $W_i$ in $G$ are either in $W_i$ or in the members of parents of $W_i$ in $G'$. Therefore, the score of DAG $G'$ equals the highest score and the algorithm returns $G'$.
\end{proof}

Note that the algorithm is guaranteed to find the equivalence class of the data-generating structure only when the compatibility condition holds. Otherwise, the found variable network may be suboptimal even if the data-generating distribution is groupwise faithful.

\section{Experiments} \label{sec:experiments}

\subsection{Implementations}

We implemented our algorithms using Matlab. The implementation is available at  \url{http://research.cs. aalto.fi/pml/software/GroupBN/}. The implementation of the PC algorithm from the BNT toolbox\footnote{\url{https://code.google.com/p/bnt/}} was used as the constraint-based version of procedure \vdag. As the score-based version, we used the state-of-the-art integer linear programming algorithm GOBNILP\footnote{\url{http://www.cs.york.ac.uk/aig/sw/gobnilp/}}.

\subsection{Simulations} \label{sec:simulation}

Next, we will evaluate the prospects of learning group DAGs in practice. Our goal is to analyze 1) to what extent it is possible to learn group DAGs from data and 2) which learning approach one should use.

We did two different simulation setups.
In {\bf Experiment 1}, we generated data from three different manually-constructed Bayesian network structures
called structures 1, 2, and 3 having $30$, $40$, and $50$ nodes,
respectively, divided into $10$ equally sized groups. All structures
were groupwise faithful to the group DAG; the network structures are shown in \ref{app:figures}. For
each structure we generated 50 binary-valued Bayesian networks by
sampling the parameters uniformly at random. Then, we sampled data
sets of size 100, 500, 2000, and 10000 from each of the Bayesian
networks.

In {\bf Experiment 2}, we randomly generated groupwise faithful structures. We are not aware of any efficient algorithm for generating groupwise faithful DAGs. Also from the experiment in the Section~\ref{sec:gf} we know that selecting both DAGs and groupings at random tend to lead complete or near-complete group DAGs. Thus, to get sparser group DAGs and variable DAGs that are groupwise faithful to them,  we used to following procedure.
\begin{itemize}
\item Fix a node set $N$ of $nk$ nodes and a grouping $W$ on $N$ with $k$ nodes in each group.
\item Generate a group DAG $H$ with $n=10$ nodes with fixed order such that each possible edge is included independently with probability $p=0.2$.
\item Select one node $w_i\in W_i$ from each group. Initialize $G$ such that $w_iw_j\in G$ if and only if $W_iW_j\in H$.
\item Repeat 1000 times
\begin{itemize}
\item Choose nodes $u$ and $v$ uniformly at random from $N$.
\item If $uv\in G$ then $G' = G\setminus \{uv\}$ else $G' = G \cup \{uv\}$.
\item If $G'$ is acyclic and $G'$ given grouping $W$ implies the same conditional independencies as $H$ then $G=G'$.
\end{itemize}
\item Return $H$ and $G$.
\end{itemize}
We generated 100 group and variable DAGs using the above procedure for group sizes $k = 2, 3, 4, 5$. Then we generated a  binary-valued Bayesian network by
sampling the parameters uniformly at random and sampled data
sets of size 100, 500, 2000, and 10000 from each of the Bayesian
networks.

We ran both the constraint-based and score-based version of
Algorithms~\ref{alg:dl} and \ref{alg:vn}. Conditional independence tests were conducted using signifance level $0.05$ and the score-based algorithms used the BDeu score with equivalent sample size $1$.
In all tests we used a 4 GB memory
limit.  As we are interested in conditional independencies, we converted DAGs into CPDAGs and measured accuracy by computing structural Hamming distance (SHD) between the data-generating CPDAG and the learned CPDAG.

\begin{figure*}[t]
	\begin{flushright}
	\includegraphics[width=1.1\textwidth]{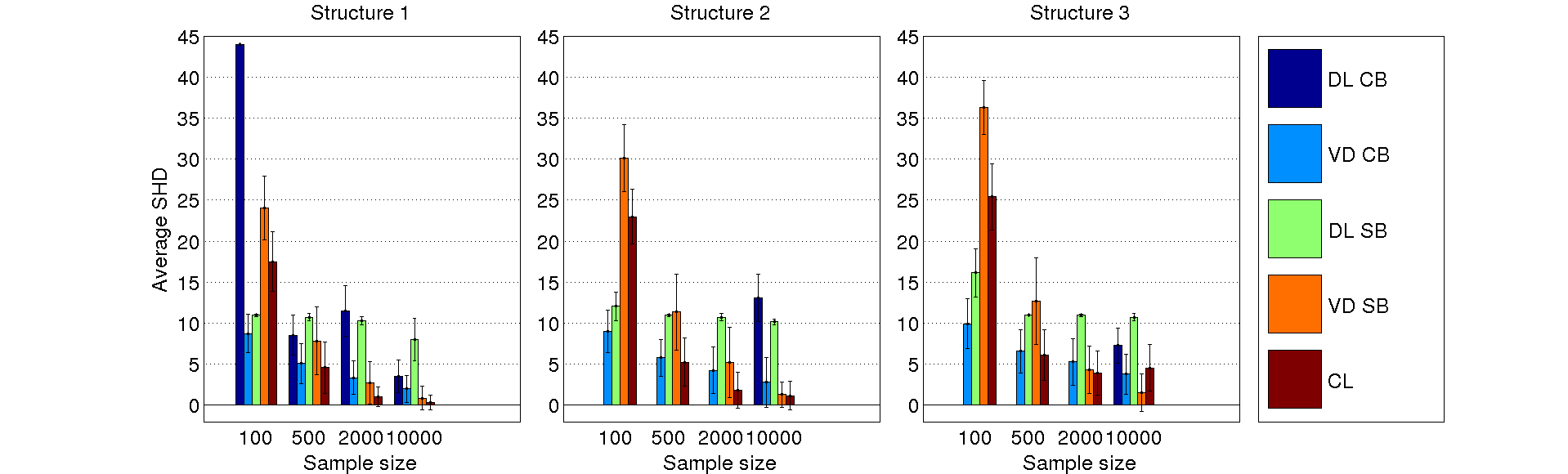}
	\end{flushright}
	\caption{Average SHD (Structural Hamming Distance) between the
          learned group CPDAG and the true group CPDAG when the data were
          generated from three different structures (Experiment 1). DL = direct learning, VD = learning using variable
          DAGs, CL = combined learning, CB = constraint-based, SB =
          score-based. The numbers on the x-axis are sample
          sizes. Missing bars for constraint-based direct learning
          are due to the algorithm running out of memory.}
	\label{fig:results2}
\end{figure*}

\begin{figure*}[t]
	\begin{flushright}
	\includegraphics[width=1.1\textwidth]{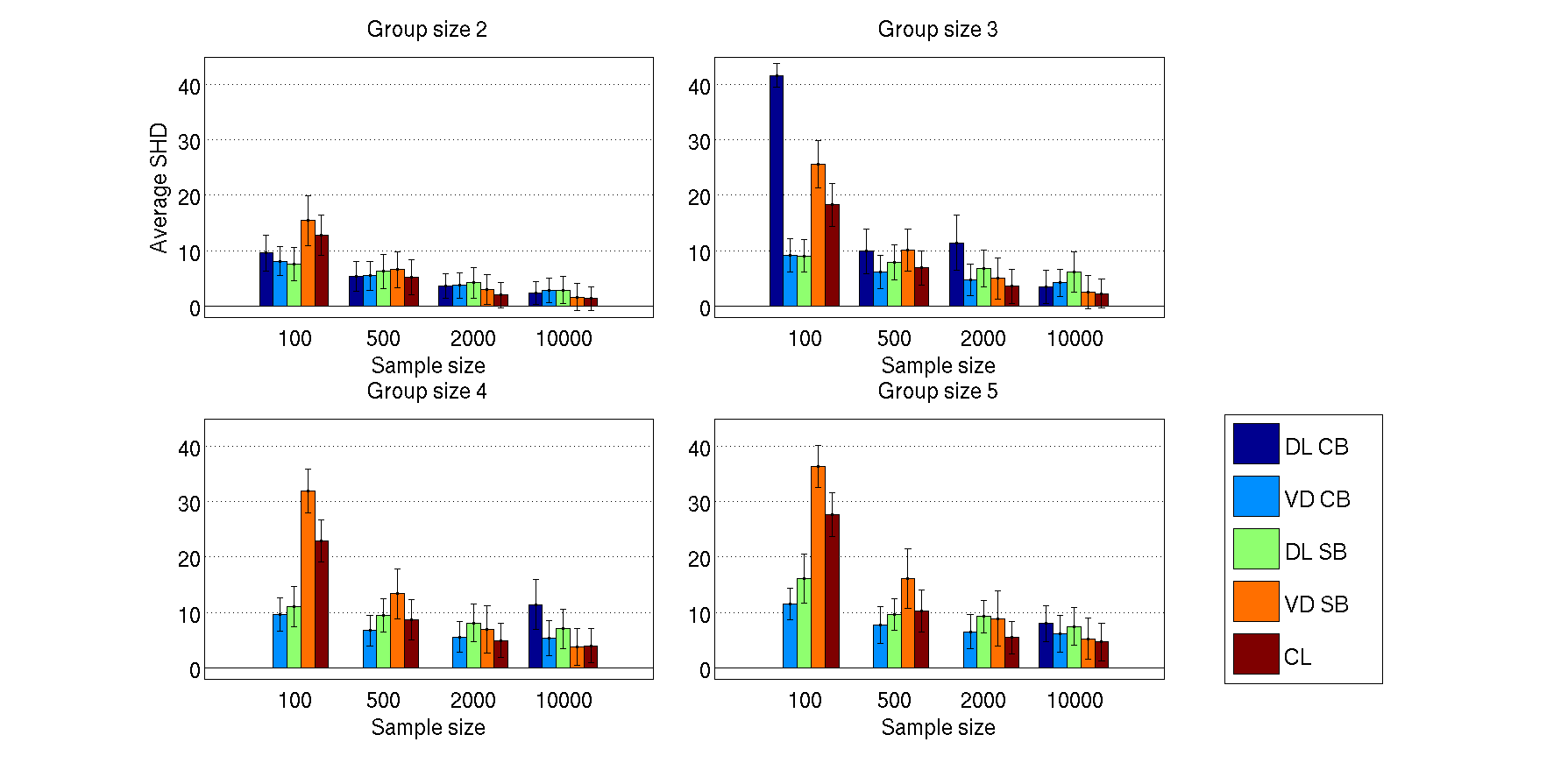}
	\end{flushright}
	\caption{Average SHD (Structural Hamming Distance) between the
          learned group CPDAG and the true group CPDAG when the data were
          generated by sampling groupwise faithful networks (Experiment 2). DL = direct learning, VD = learning using variable
          DAGs, CL = combined learning, CB = constraint-based, SB =
          score-based. The numbers on the x-axis are sample
          sizes. Missing bars for constraint-based direct learning
          are due to the algorithm running out of memory.}
	\label{fig:results_new}
\end{figure*}

The results from Experiments~1 and 2 are shown in Figures~\ref{fig:results2} and \ref{fig:results_new}, respectively. To answer our first research question, we notice that both experiments suggest that group DAGs can be learned accurately when the groups are small and there are sufficiently many samples; see, e.g., Figure~\ref{fig:results_new} with group size 2 and 10000 samples. However, the accuracy seems to decrease when the group size grows or the number of samples decreases. Intuitively, the decrease of accuracy when the groups size grows makes sense because the bigger the groups the more possibilities there are to add false positive edges to the group DAG. 

We also observe that constraint-based direct learning struggles often and in many cases we do not get any results because the algorithm runs out of memory. This is due to the fact that variables in the direct learning approach have lots of states and thus direct learning requires lots of data to draw any conclusions. On the other hand, it seems that the constraint-based lerning via variable DAGs performs well. Especially, it is generally the most accurate approach when there are few samples. The relatively good performance of the constraint-based approach when there is little data can be explained at least partially as follows. Intuitively, learning a true positive edge in the group DAG is robust: To include a true positive edge, it is enough that the learned variable DAG preserves only one d-connected path between the groups (out of possibly many such paths). On the other hand, even one false positive dependence between two nodes in different groups leads to connecting the two groups in the group DAG. Thus, too sparse variable DAGs seem to result in more accurate group DAGs than too dense variable DAGs.
This intuition is supported by our empirical observation that typically, learned group DAGs have more false positive edges than false negatives. Furthermore, we observe that constraint-based methods tend to be more conservative, that is, if there is little data then the variable DAG learned with the constraint-based method tends to be sparser than the variable DAG learned with the score-based method; the sparsity may be due to type II errors in conditional independence tests. 

Furthermore, we observe that the accuracy of score-based direct learning is not significantly affected by the sample size. Score-based learning via variables DAGs is very accurate when there are lots of samples. However, its accuracy decreases substantially if the number of samples is low.

Also combined learning gave accurate results, especially when the sample size was large, although all other methods have better theoretical guarantees than combined learning. Combined learning forces the topological orders of the variable and group DAG to be compatible and this might act as some kind of implicit regularization. Note that in Experiment~1 combined learning benefits from the fact that the topological orders of the data-generating variable and group DAGs were compatible but it was still quite accurate in Experiment~2 were the topological orders were not always compatible.

To answer our second question, we conclude that constraind-based learning via variable DAGs is the most accurate method if there are only few (less than 500) samples. If there are plenty samples then combined learning and score-based learning via variable DAGs are the most accurate approaches.

\subsection{Real data} \label{sec:real_data}

Next, we demonstrate learning of group DAGs from real data and challenges that are faced in this scenario. A prominent challenge here is the difficulty of assessing the quality of the learned group DAGs in the absence of ground-truth.

We applied the learning methods to the {\sc Housing} data that is
available at the UCI machine learning repository \cite{bache13}. The data
contain 14 variables for 506 observations, measuring multiple factors
affecting housing prices in different neighborhoods in the Boston
area. We grouped the variables into 9 groups. Group {\em
  Accessibility} consisted of variables CHAS, DIS, and RAD, group {\em
  Zoning} consisted of variables ZN and INDUS, group {\em Apartment
  properties} consisted of variables RM and AGE, and group {\em
  Population} consisted of variables B and LSTAT. Five of our groups
consisted of one variable: {\em Crime} of CRIM, {\em Pollution} of
NOX, {\em Education} of PTRATIO, {\em House prices} of MEDV, and {\em
  Taxes} of TAX.

We learned a group DAG using each of the five algorithms; all group
DAGs (as well as corresponding variable DAGs when applicable) are
shown in \ref{app:figures}. We show here
only two representative networks. Our simulations (Section~\ref{sec:simulation}) showed that
constraint-based learning via a variable DAG and combined learning resulted in smallest average SHD with sample
size 500 so we chose them as representative methods; the group DAG from constraint-based learning is shown in
Figure~\ref{fig:housing2}(a) and the corresponding variable DAG in
Figure~\ref{fig:housing2}(b). The DAGs from combined learning are shown in Figure~\ref{fig:housing5}.

\begin{figure}[hbt!]
\begin{center}
\hspace{-0.3in}
\begin{tabular}{cc}
\includegraphics[width=0.5\linewidth]{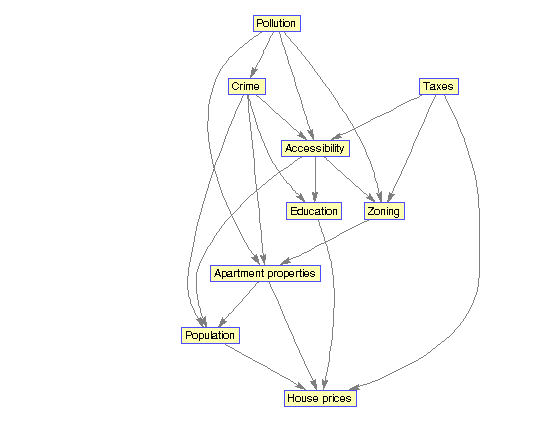} &
\includegraphics[width=0.5\linewidth]{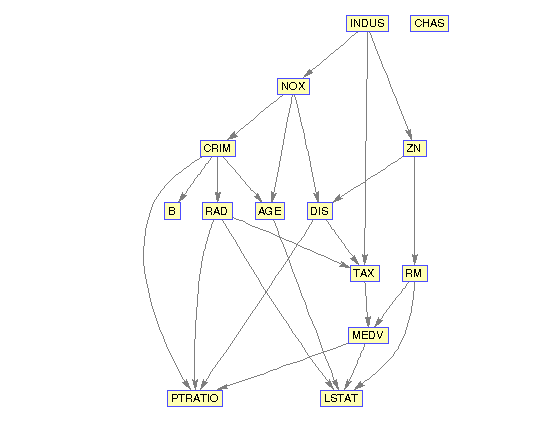} \\
(a) &
(b)\\
\end{tabular}
\caption{(a) The group DAG learned from {\sc Housing} data using constraint-based learning via a variable DAG. (b) The corresponding variable DAG.
\label{fig:housing2}
}
\end{center}
\end{figure}

\begin{figure}[hbt!]
\begin{center}
\hspace{-0.3in}
\begin{tabular}{cc}
\includegraphics[width=0.5\linewidth]{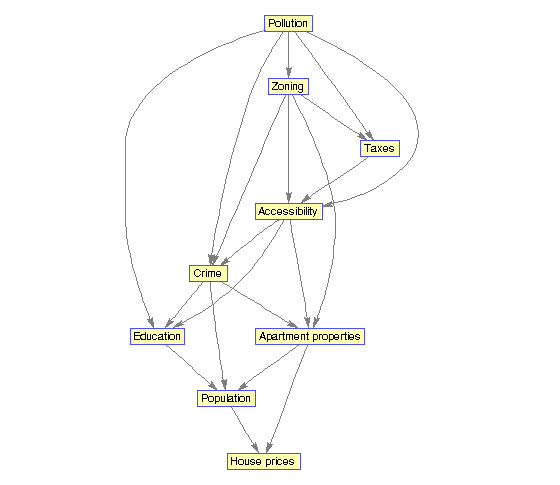} &
\includegraphics[width=0.5\linewidth]{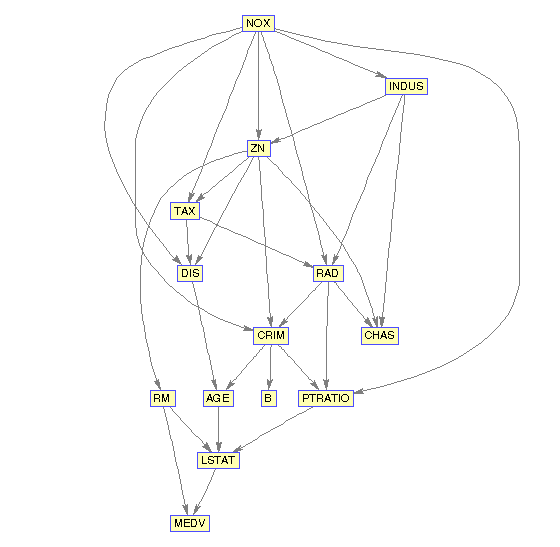} \\
(a) &
(b)\\
\end{tabular}
\caption{(a) The group DAG learned from {\sc Housing} data using combined learning. (b) The corresponding variable DAG.
\label{fig:housing5}
}
\end{center}
\end{figure}

We can make several observations from Figure~\ref{fig:housing2}. 
We notice that the group DAG in Figure~\ref{fig:housing2}(a) has a v-structure {\em Apartment properties} $\rightarrow$ {\em House prices} $\leftarrow$ {\em Taxes}. By Theorem~\ref{thm:weak_v}, this implies that {\em Apartment properties} and {\em Taxes} are group causes of  {\em House prices} and thus manipulating them would affect house prices; this seems  a plausible conclusion. We see from the variable DAG that, indeed, there are directed paths from both {\em Apartment properties} and {\em Taxes} to {\em House prices}. However, the variable DAG shown in Figure~\ref{fig:housing2}(b) is not groupwise faithful to the group DAG given the grouping. To see this, we notice that {\em Zoning} and {\em Crime} are conditionally independent in the variable DAG given {\em Pollution} and {\em Apartment properties} but not in the group DAG. Thus, the group DAG expresses some dependencies that are not present in the variable DAG.

We see that the DAGs in Figure~\ref{fig:housing5} differ from the ones in Figure~\ref{fig:housing2}. For example, {\em House prices} is a neighbor of {\em Apartment properties} but not with {\em Taxes} in the group DAG. Overall, structural Hamming distance between the group DAGs is $19$. While this case study is not enough to warrant any statistical conclusions, we recommend one not to trust blindly on the learned group DAGs because results may be sensitive to the choice of an algorithm.

\section{Discussion}

In this paper we introduced the concept of group DAG for modeling
conditional independencies and dependencies between groups of random
variables, and studied prospects of learning group DAGs. It turned out,
perhaps unsurprisingly, that many aspects become more complicated when
moving from individual variables to groups of variables. We 
showed that in order to have theoretical guarantees for the quality of
learned networks, one has to assume groupwise faithfulness, which is a
rather strong assumption. Further, inferring causal relationships
between groups becomes more tricky.

In this paper, we studied structure learning. Naturally, it is possible to extend group DAGs to group Bayesian networks by learning parameters. As each group can be treated as a variable, we can use any standard method for learning parameters.
However, it should be noted that the group variables tend to have lots of states which may render the estimation of parameters inaccurate. Therefore, if the goal is to use the learned network to infer probabilities then one may want to use a standard Bayesian network instead of a group Bayesian network. 

Our experiments suggest that data does not always behave ``nicely''. One inevitable difficulty is that data are often groupwise unfaithful. The other practical challenge is that principled methods add an edge to the group DAG if there exists even one weak dependency between two groups. Therefore, erroneous dependencies from conditional independence tests or local scores can lead into lots of false positive edges in the group DAG. In practice, it may be desirable to take a less principled approach and use some kind of regularization to get rid of spurious edges. One way to alleviate this problem is to use a low significance level in the conditional independence tests.

We have assumed that the variable groups are known beforehand, as prior knowledge, and asked what can be done with the extra prior knowledge. A natural follow-up question is that can the groups be learned from data. Even though this interesting question is superficially related it is, however, a distinct and very different problem that is likely to require a different machinery. Multiple different goals for such a clustering of variables are possible and sensible.

\section*{Acknowledgements}

The authors thank Cassio de Campos, Antti Hyttinen, Esa Junttila, Jefrey Lijffijt, Daniel Malinsky, Teemu Roos, and Milan Studen\'y for useful discussions. The work was partially funded by The Academy of Finland (Finnish Centre of Excellence in Computational Inference Research COIN). The experimental results were computed using computer resources within the Aalto University School of Science "Science-IT" project.

\FloatBarrier
\section*{References}
\bibliographystyle{plain}
\bibliography{ref}

\begin{thebibliography}{10}

\bibitem{aliferis10a}
C.F. Aliferis, A.~Statnikov, I.~Tsamardinos, S.~Mani, and X.D. Koutsoukos.
\newblock {Local Causal and Markov Blanket Induction for Causal Discovery and
  Feature Selection for Classification Part I: Algorithms and Empirical
  Evaluation}.
\newblock {\em Journal of Machine Learning Research}, 11:171--234, 2010.

\bibitem{bache13}
K.~Bache and M.~Lichman.
\newblock {UCI} machine learning repository, 2013.

\bibitem{burge06}
J.~Burge and T.~Lane.
\newblock {Improving Bayesian Network Structure Search with Random Variable
  Aggregation Hierarchies}.
\newblock In {\em ECML}, pages 66--77. Springer, Berlin, Heidelberg, 2006.

\bibitem{chickering96}
D.M. Chickering.
\newblock {Learning Bayesian networks is NP-Complete}.
\newblock In {\em Learning from Data: Artificial Intelligence and Statistics},
  pages 121--130. Springer-Verlag, 1996.

\bibitem{chickering02b}
D.M. Chickering.
\newblock {Optimal Structure Identification With Greedy Search}.
\newblock {\em Journal of Machine Learning Reseach}, 3:507--554, 2002.

\bibitem{cooper92}
G.F. Cooper and E.~Herskovits.
\newblock {A Bayesian Method for the Induction of Probabilistic Networks from
  Data}.
\newblock {\em Machine Learning}, 9(4):309--347, 1992.

\bibitem{cussens11}
J.~Cussens.
\newblock {Bayesian network learning with cutting planes}.
\newblock In {\em UAI}, pages 153--160. AUAI Press, 2011.

\bibitem{davis99}
S.~Davies and A.~Moore.
\newblock Bayesian network for lossless dataset compression.
\newblock In {\em Proceedings of the fifth ACM SIGKDD international conference
  on Knowledge discovery and data mining (KDD)}, pages 387--391, 1999.

\bibitem{entner12}
D.~Entner and P.O. Hoyer.
\newblock {Estimating a Causal Order among Groups of Variables in Linear
  Models}.
\newblock In {\em ICANN}, pages 83--90. Spinger, 2012.

\bibitem{erdos59}
P.~Erd\H{o}s and A.~R\'enyi.
\newblock On random graphs i.
\newblock {\em Publicationes Mathematicae}, 6:290--–297, 1959.

\bibitem{gilbert59}
E.N. Gilbert.
\newblock Random graphs.
\newblock {\em Annals of Mathematical Statistics}, 30:1141--–1144, 1959.

\bibitem{gyftodimos02}
E.~Gyftodimos and P.A. Flach.
\newblock {Hierarchical Bayesian networks: a probabilistic reasoning model for
  structured domains}.
\newblock In {\em ICML-2002 Workshop on Development of Representations}, 2002.

\bibitem{heckerman95}
D.~Heckerman, D.~Geiger, and D.M. Chickering.
\newblock {Learning Bayesian Networks: The Combination of Knowledge and
  Statistical Data}.
\newblock {\em Machine Learning}, 20(3):197--243, 1995.

\bibitem{jaakkola10}
T.~Jaakkola, D.~Sontag, A.~Globerson, and M.~Meila.
\newblock {Learning Bayesian Network Structure using LP Relaxations}.
\newblock In {\em AISTATS}, pages 358--365, 2010.

\bibitem{koller97}
D.~Koller and A.~Pfeffer.
\newblock {Object-oriented Bayesian networks}.
\newblock In {\em UAI}, pages 302--313. Morgan Kaufmann Publishers Inc., 1997.

\bibitem{meek95b}
C.~Meek.
\newblock {Causal Inference and Causal Explanation with Background Knowledge}.
\newblock In {\em UAI}, pages 403--410. Morgan Kaufmann, 1995.

\bibitem{parviainen16}
P.~Parviainen and S.~Kaski.
\newblock Bayesian networks for variable groups.
\newblock In {\em JMLR: Workshop and Conference Proccedings}, volume~52, pages
  380--391, 2016.

\bibitem{pearl00}
J.~Pearl.
\newblock {\em {Causality: Models, Reasoning, and Inference}}.
\newblock Cambridge university Press, 2000.

\bibitem{segal05}
E.~Segal, D.~Pe'er, A.~Regev, D.~Koller, and N.~Friedman.
\newblock {Learning Module Networks}.
\newblock {\em Journal of Machine Learning Reseach}, 6:557--588, October 2005.

\bibitem{silander06}
T.~Silander and P.~Myllym{\"a}ki.
\newblock {A simple approach for finding the globally optimal Bayesian network
  structure}.
\newblock In {\em UAI}, pages 445--452. AUAI Press, 2006.

\bibitem{slobodianik09}
N.~Slobodianik, D.~Zaporozhets, and N.~Madras.
\newblock Strong limit theorems for the bayesian scoring criterion in bayesian
  networks.
\newblock {\em Journal of Machine Learning Research}, 10:1511--1526, 2009.

\bibitem{spirtes00}
P.~Spirtes, C.~Glymour, and R.~Scheines.
\newblock {\em {Causation, Prediction, and Search}}.
\newblock Springer Verlag, 2000.

\bibitem{verma90}
T.S. Verma and J.~Pearl.
\newblock Equivalence and synthesis of causal models.
\newblock In {\em UAI}, pages 255--270. Elsevier, 1990.

\bibitem{xiang93}
Y.~Xiang, D.~Poole, and M.~P. Beddoes.
\newblock Multiply sectioned bayesian networks and junction forests for large
  knowledge-based systems.
\newblock {\em Computational Intelligence}, 9:171--220, 1993.

\end{thebibliography}

\pagebreak
\appendix
\section{Proofs of Theorems~\ref{thm:imply} and \ref{thm:imply2}} \label{app:proofs}

Next, we will prove Theorems~\ref{thm:imply} and \ref{thm:imply2}. We will start by proving some lemmas that are used in the proof of Theorem~\ref{thm:imply}.

\begin{lemma} \label{thm:size_zero}
Let $H$ be a group DAG on a grouping $W$ and let $\max_{i} |W_i| = 1$. Then any distribution $p$ on $\cup W_i$ that is groupwise faithful to $H$ given $W$ is faithful to some variable DAG on $\cup W_i$.
\end{lemma}
\begin{proof}
As all groups consist of exactly one variable, the conditional independencies implied by the group DAGs has to be  expressed exactly by the data-generating distribution, that is, the variable DAG  (up to a relabelling). Thus, the data-generating distribution $p$ has to be faithful to a DAG.
\end{proof}

\begin{lemma} \label{thm:no_neighbors}
Let $H$ be a group DAG on grouping $W$ and let $\max_{i} |W_i|  = 2$. If no group of size 2 has neighbors, then all distributions on $\cup W_i$ that are groupwise faithful to $H$ given $W$ are faithful to a variable DAG.
\end{lemma}
\begin{proof}
Clearly, none of the members of the groups of size 2 cannot be connected to any variables outside the group. The two variables inside a group are either independent or dependent. In both cases their joint distribution is faithful to a DAG. 

By Lemma~\ref{thm:size_zero}, the variable DAG corresponding to the subgraph of the group DAG induced by the groups of size 1 is faithful to a DAG. Thus, the distribution $p$ is faithful to a DAG
\end{proof}

Now, we are ready to prove Theorem~\ref{thm:imply} which follows straightforwardly from the previous lemmas.

\thmimply*
\begin{proof}
Follows directly from Lemmas~\ref{thm:size_zero} and \ref{thm:no_neighbors}.
\end{proof}

Next, we will prove Theorem~\ref{thm:imply2}. We start by proving two lemmas.

In the following proofs we will exploit the well-known fact that an exclusive or (XOR) distribution is unfaithful. That is, if we have three binary variables $X$, $Y$, and $Z$ where $P(X = 1) = P(Y = 1) = 1/2$ and $Z= \text{XOR}(X, Y)$ then the conditional independencies cannot be expressed exactly using any DAG. To see this, we note that $Z$ depends on both $X$ and $Y$. However, it is marginally independent of both of them.

\begin{lemma} \label{thm:large_neighbors}
Let $H$ be a group DAG on grouping $W$ and let $\max_{i} |W_i|  = 2$. If two groups of size 2 are neighbors, then not all distributions on $\cup W_i$ that are groupwise faithful to $H$ given $W$ are faithful to a variable DAG.
\end{lemma}
\begin{proof}
It suffices to show that for any group DAG--grouping pair there exists a distribution $p$ that implies exactly the same groupwise conditional independencies as $H$ given $W$ but $p$ is not faithful to any variable DAG.

Without loss of generality, let us assume that $|W_1| = |W_2| = 2$ and $W_1$ is a parent of $W_2$  in the group DAG. Further, let $w_i\in W_i$ be a specified element of a group. Now let us construct a variable DAG $G$ as follows. If there is an arc from $W_i$ to $W_j$ in $H$ then there is an arc from $w_i$ to $w_j$ in $G$. Further, there are arcs $uv$ and $w_2v$ in $G$, where $u\in W_1\setminus \{w_1\}$ and $v\in W_2\setminus \{w_2\}$. If we choose parameters such that the marginal distribution on $\cup W_i \setminus \{u, v\}$ is faithful to the induced subgraph $G[\cup W_i \setminus \{u, v\}]$ and the local conditional distribution of node $v$ is an exclusive or (XOR) distribution, then the distribution $p$ expresses exactly the same groupwise conditional independencies as $H$ but is not faithful to any DAG.
\end{proof}

\begin{lemma} \label{thm:size_three}
Let $H$ be a group DAG on grouping $W$ and let $\max_{i} |W_i| \geq 3$. Then not all distributions on $\cup W_i$ that are groupwise faithful to $H$ given $W$ are faithful to a variable DAG.
\end{lemma}
\begin{proof}
It is enough to show that for any group DAG--grouping pair there exists a distribution $p$ that implies exactly the same conditional independencies as $H$ but $p$ is not faithful to any variable DAG.

Without loss of generality, let us assume that $|W_1| \geq 3$. Further, let $w_i\in W_i$ be a specified element of a group. Now let us construct a variable DAG $G$ as follows. If there is an arc from $W_i$ to $W_j$ in $H$ then there is an arc from $w_i$ to $w_j$ in $G$. Further, there are arcs $w_1u$ and $vu$ in $G$, where $u,v\in W_1$. If we choose parameters such that the marginal distribution on $\cup W_i \setminus \{u, v\}$ is faithful to the induced subgraph $G[\cup W_i \setminus \{u, v\}]$ and the local conditional distribution of the node $u$ is an exclusive or (XOR) distribution, then the distribution $p$ expresses exactly the same groupwise conditional independencies as $H$ but is not faithful to any DAG.
\end{proof}

We are ready to prove Theorem~\ref{thm:imply2}.

\thmimplyb*
\begin{proof}
Follows directly from Lemmas~\ref{thm:large_neighbors} and \ref{thm:size_three}.
\end{proof}

\pagebreak
\section{Additional figures} \label{app:figures}

\begin{figure}[hbt!]
\begin{center}
\hspace{-0.3in}
\begin{tabular}{cc}
\includegraphics[width=0.42\linewidth]{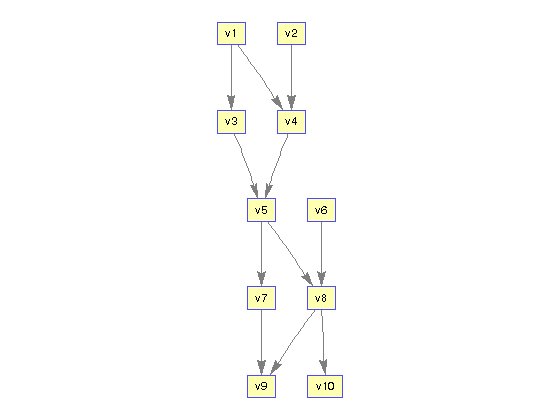}  &
\includegraphics[width=0.42\linewidth]{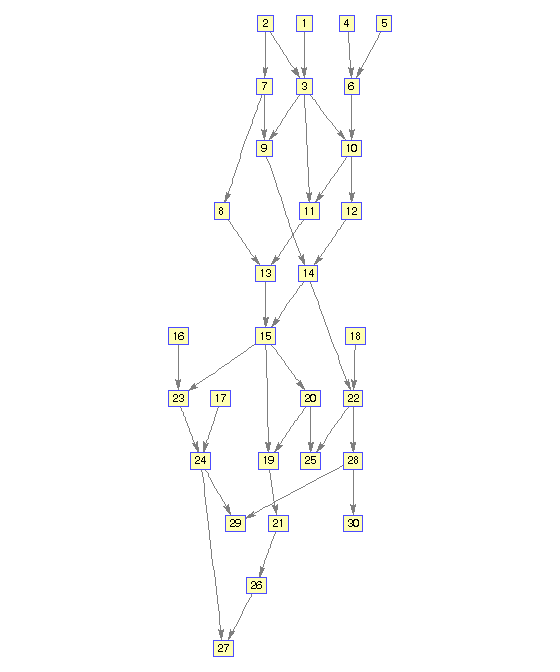}   \\
(a) & (b)\\
\includegraphics[width=0.42\linewidth]{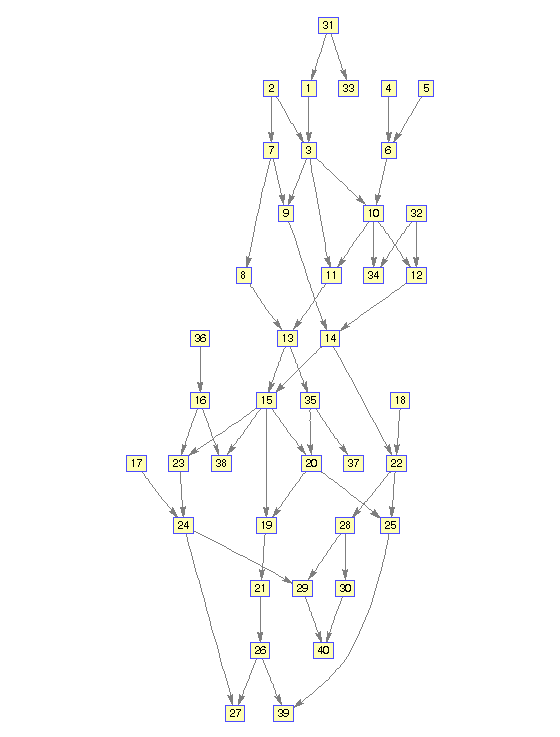}  &
\includegraphics[width=0.42\linewidth]{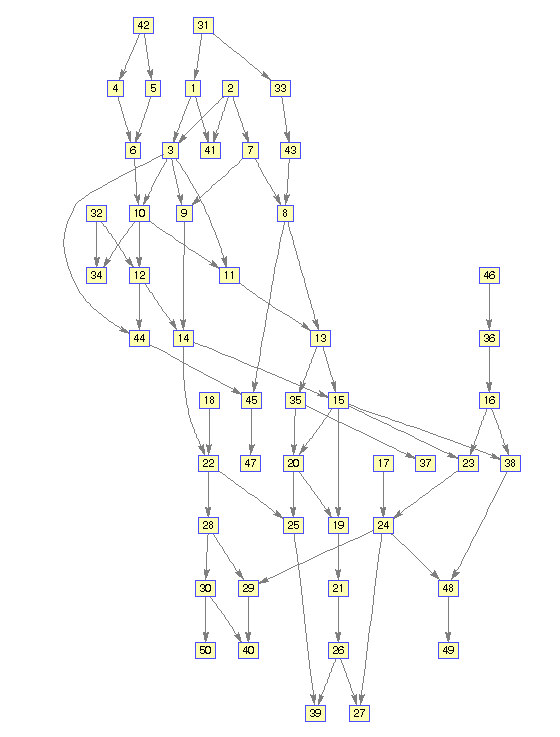} \\
(c) & (d) 
\end{tabular}
\caption{(a) The data-generating group DAG. (b) The variable DAG of the first data-generating structure. Groups are following: $v_1 = \{1, 2, 3\}$, $v_2 = \{4, 5, 6\}$, $v_3 = \{7, 8, 9\}$,  $v_4 = \{10, 11, 12\}$, $v_5 = \{13, 14, 15\}$, $v_6 = \{16, 17, 18\}$, $v_7 = \{19, 20, 21\}$, $v_8 = \{22, 23, 24\}$, $v_9 = \{25, 26, 27\}$, and $v_{10} = \{ 28, 29, 30\}$. (c) The variable DAG of the second data-generating structure. Groups are following: $v_1 = \{1, 2, 3, 31\}$, $v_2 = \{4, 5, 6, 32\}$, $v_3 = \{7, 8, 9, 33\}$,  $v_4 = \{10, 11, 12, 34\}$, $v_5 = \{13, 14, 15, 35\}$, $v_6 = \{16, 17, 18, 36\}$, $v_7 = \{19, 20, 21, 37\}$, $v_8 = \{22, 23, 24, 38\}$, $v_9 = \{25, 26, 27, 39\}$, and $v_{10} = \{ 28, 29, 30, 40\}$.  (d) The variable DAG of the third data-generating structure. Groups are following: $v_1 = \{1, 2, 3, 31, 41\}$, $v_2 = \{4, 5, 6, 32, 42\}$, $v_3 = \{7, 8, 9, 33, 43\}$,  $v_4 = \{10, 11, 12, 34, 44\}$, $v_5 = \{13, 14, 15, 35, 45\}$, $v_6 = \{16, 17, 18, 36, 46\}$, $v_7 = \{19, 20, 21, 37, 47\}$, $v_8 = \{22, 23, 24, 38, 48\}$, $v_9 = \{25, 26, 27, 39, 49\}$, and $v_{10} = \{ 28, 29, 30, 40, 50\}$.   
}
\end{center}
\end{figure}

\begin{figure*}[hbt!]
	\centering
	\begin{tabular}{cc}
	\includegraphics[width=0.49\textwidth]{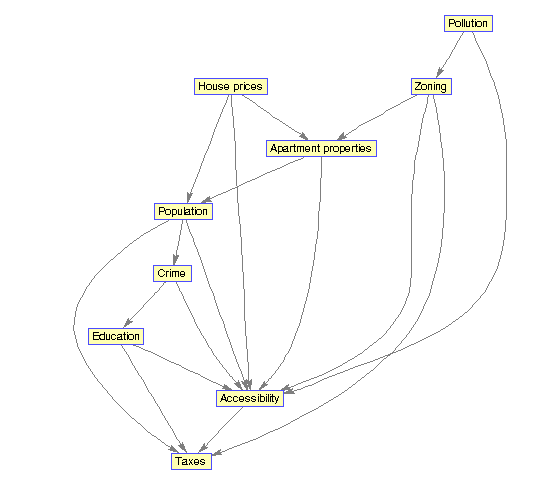} &
			\includegraphics[width=0.49\textwidth]{housing2}\\
			(a) & (b) \\
				\includegraphics[width=0.49\textwidth]{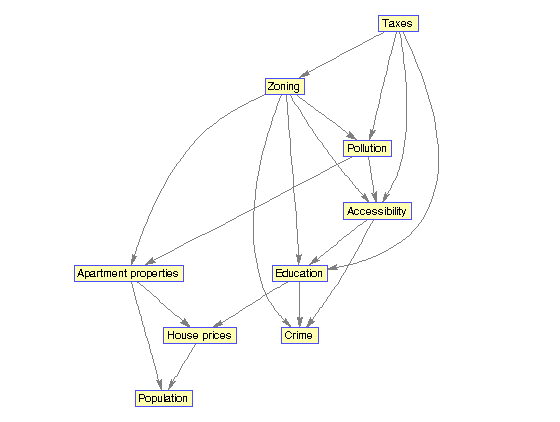} &
					\includegraphics[width=0.49\textwidth]{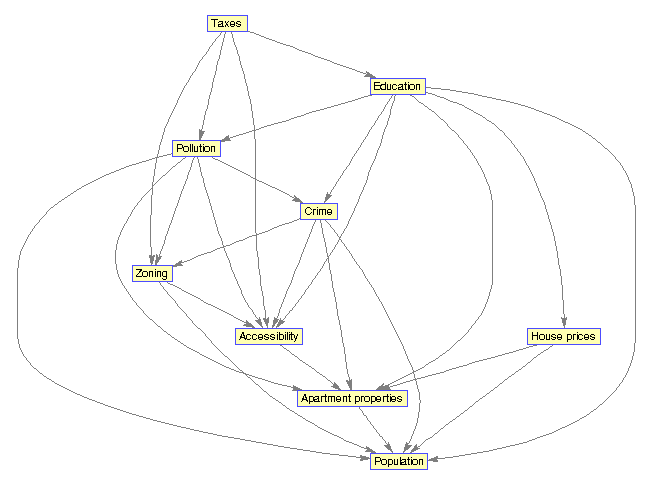} \\
					(c) & (d)\\
						\includegraphics[width=0.49\textwidth]{housing5} & \\
						(e) &
	\end{tabular}
	\caption{The group DAG learned from {\sc Housing} data using (a) constraint-based direct learning, (b) constraint-based learning via variable DAG, (c) score-based direct learning, (d) score-based learning via variable DAG, and (e) combined learning.}
\end{figure*}

\begin{figure*}[hbt!]
	\centering
	\begin{tabular}{cc}
	\includegraphics[width=0.5\textwidth]{housing6} &
	\includegraphics[width=0.5\textwidth]{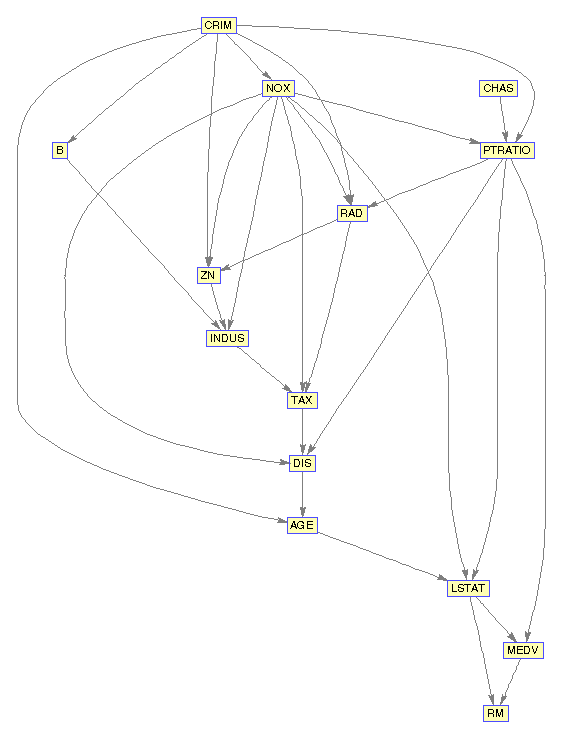} \\
	(a) & (b) \\
	\includegraphics[width=0.5\textwidth]{housing8} & \\
  (c) &
	\end{tabular}
	\caption{The variable DAG learned from {\sc Housing} data using (a) constraint-based learning, (b)  score-based learning, and (c) combined learning.}	
\end{figure*}

\end{document}